\PassOptionsToPackage{table}{xcolor}
\documentclass[sigconf]{acmart}

\copyrightyear{2020}
\acmYear{2020}
\setcopyright{acmcopyright}\acmConference[KDD '20]{Proceedings of the 26th ACM SIGKDD Conference on Knowledge Discovery and Data Mining}{August 23--27, 2020}{Virtual Event, CA, USA}
\acmBooktitle{Proceedings of the 26th ACM SIGKDD Conference on Knowledge Discovery and Data Mining (KDD '20), August 23--27, 2020, Virtual Event, CA, USA} \acmPrice{15.00}
\acmDOI{10.1145/3394486.3403224}
\acmISBN{978-1-4503-7998-4/20/08}

\usepackage{colortbl}
\usepackage{xcolor}
\usepackage{dirtytalk}
\usepackage{float}
\usepackage[matit]{echodefs}
\usepackage{cleveref}
\newtheorem{problem}{Problem}
\usepackage{thmtools, thm-restate}
\usepackage{algorithm,algpseudocode}
\usepackage{todonotes}
\usepackage[bottom]{footmisc}
\usepackage{footnote}
\makesavenoteenv{tabular}
\makesavenoteenv{table}
\usepackage{array}
\usepackage{multirow}
\definecolor{Gray}{gray}{0.9}

\pdfpageattr {/Group << /S /Transparency /I true /CS /DeviceRGB>>} 

\newcommand\sbullet[1][.5]{\mathbin{\vcenter{\hbox{\scalebox{#1}{$\bullet$}}}}}

\newcommand{\etal}{\textit{et al. }}

\def\bL{{\mathbb L}}
\def\bH{{\mathbb H}}
\def\bJ{{\mathbb J}}
\def\bI{{\mathbb I}}
\def\bK{{\mathbb K}}

\AtBeginDocument{%
  \providecommand\BibTeX{{%
    \normalfont B\kern-0.5em{\scshape i\kern-0.25em b}\kern-0.8em\TeX}}}

\begin{document}
 \fancyhead{}
\title{Hyperbolic Distance Matrices}



\author{Puoya Tabaghi}
\orcid{0000-0002-1914-5950}
\affiliation{%
 \institution{University of Illinois at Urbana-Champaign}
}
\email{tabaghi2@illinois.edu}

\author{Ivan Dokmani\'c}
\orcid{0000-0001-7132-5214}
\affiliation{%
 \institution{University of Basel}
}
\email{ivan.dokmanic@unibas.ch}


\begin{abstract}
Hyperbolic space is a natural setting for mining and visualizing data with hierarchical structure. In order to compute a hyperbolic embedding from comparison or similarity information, one has to solve a hyperbolic distance geometry problem. In this paper, we propose a unified framework to compute hyperbolic embeddings from an arbitrary mix of noisy metric and non-metric data. Our algorithms are based on semidefinite programming and the notion of a hyperbolic distance matrix, in many ways parallel to its famous Euclidean counterpart. A central ingredient we put forward is a semidefinite characterization of the hyperbolic Gramian---a matrix of Lorentzian inner products. This characterization allows us to formulate a semidefinite relaxation to efficiently compute hyperbolic embeddings in two stages: first, we complete and denoise the observed hyperbolic distance matrix; second, we propose a spectral factorization method to estimate the embedded points from the hyperbolic distance matrix. We show through numerical experiments how the flexibility to mix metric and non-metric constraints allows us to efficiently compute embeddings from arbitrary data.
\end{abstract}



\maketitle
\section{INTRODUCTION} \label{sec:intro}

Hyperbolic space is roomy. It can embed hierarchical structures uniformly and with arbitrarily low distortion \cite{lamping1994laying, sarkar2011low}. Euclidean space cannot achieve comparably low distortion even using an unbounded number of dimensions \cite{linial1995geometry}. 

Embedding objects in hyperbolic spaces has found a myriad applications in exploratory science, from visualizing hierarchical structures such as social networks and link prediction for symbolic data \cite{verbeek2014metric,nickel2017poincare} to natural language processing \cite{dhingra2018embedding,le2019inferring}, brain networks \cite{cannistraci2013link}, gene ontologies \cite{ashburner2000gene} and recommender systems \cite{vinh2018hyperbolic, chamberlain2019scalable}. 

Commonly in these applications, there is a tree-like data structure which encodes \textit{similarity} between a number of entities. We experimentally observe some relational information about the structure and the data mining task is to find a geometric representation of the entities consistent with the experimental information. In other words, the task is to compute an embedding. This concept is closely related to the classical distance geometry problems and multidimensional scaling (MDS) \cite{kruskal1978multidimensional} in Euclidean spaces \cite{liberti2014euclidean, dokmanic2015euclidean}. 

The observations can be metric or non-metric. Metric observations convey (inexact) distances; for example, in internet distance embedding a small subset of nodes with complete distance information are used to estimate the remaining distances \cite{shavitt2008hyperbolic}. Non-metric observations tell us which pairs of entities are closer and which are further apart. The measure of closeness is typically derived from domain knowledge; for example, word embedding algorithms aim to relate semantically close words and their topics \cite{mikolov2013distributed,pennington2014glove}.  

In scientific applications it is desirable to compute good low-dimensional hyperbolic embeddings. Insisting on low dimension not only facilitates visualization, but also promotes simple explanations of the phenomenon under study. However, in most works that leverage hyperbolic geometry the embedding technique is not the primary focus and the related computations are often ad hoc. The situation is different in the Euclidean case, where the notions of MDS, Euclidean distance matrices (EDMs) and their characterization in terms of positive semidefinite Gram matrices play a central role in the design and analysis of algorithms \cite{liberti2014euclidean,alfakih1999solving}.

In this paper, we focus on computing low-dimensional hyperbolic embeddings. While there exists a strong link between Euclidean geometry and positive (semi)definiteness, we prove that what we call \emph{hyperbolic distance matrices} (HDMs) can also be characterized via semidefinite constraints. Unlike in the Euclidean case, the hyperbolic analogy of the Euclidean Gram matrix is a linear combination of two rank-constrained semidefinite variables. Together with a spectral factorization method to directly estimate the hyperbolic points, this characterization gives rise to flexible embedding algorithms which can handle diverse constraints and mix metric and non-metric data. 

\subsection{Related Work}
The usefulness of hyperbolic space stems from its ability to efficiently represent the geometry of complex networks \cite{asta2014geometric,krioukov2010hyperbolic}. Embedding metric graphs with underlying hyperbolic geometry has applications in word embedding \cite{mikolov2013distributed,pennington2014glove}, geographic routing  \cite{kleinberg2007geographic}, routing in dynamical graphs \cite{cvetkovski2009hyperbolic}, odor embedding \cite{zhou2018hyperbolic}, internet network embedding for delay estimation and server selection \cite{shavitt2008hyperbolic,boguna2010sustaining}, to name a few. In the literature such problems are known as hyperbolic multidimensional scaling \cite{de2018representation}. 

There exist Riemann gradient-based approaches \cite{chowdhary2018improved,nickel2018learning,nickel2017poincare,le2019inferring} which can be used to directly estimate such embeddings from metric measurements \cite{roller2018hearst}. We emphasize that these methods are iterative and only guaranteed to return a locally optimal solution. 
On the other hand, there exist one-shot methods to estimate hyperbolic embeddings from a \textit{complete} set of measured distances. The method of Wilson \etal \cite{wilson2014spherical} is based on spectral factorization of an inner product matrix (we refer to it as hyperbolic Gramian) that directly minimizes a suitable \textit{stress}. In this paper, we derive a semidefinite relaxation to estimate the \textit{missing} measurements and denoise the distance matrix, and then follow it with the spectral embedding algorithm.

Non-metric (or order) embedding has been proposed to learn visual-semantic hierarchies from ordered input pairs by embedding symbolic objects into a low-dimensional space \cite{vendrov2015order}. In the Euclidean case, stochastic triplet embeddings \cite{van2012stochastic}, crowd kernels \cite{tamuz2011adaptively}, and generalized non-metric MDS \cite{agarwal2007generalized} are some well-known order embedding algorithms. For embedding hierarchical structures, Ganea \etal \cite{ganea2018hyperbolic} model order relations as a family of nested geodesically convex cones. Zhou \textit{et. al.} \cite{zhou2018hyperbolic} show that odors can be efficiently embedded in hyperbolic space provided that the similarity between odors is based on the statistics of their co-occurrences within natural  mixtures.

\subsection{Contributions} 

We summarize our main contributions as follows:
\begin{itemize}
\item {\bf Semidefinite characterization of HDMs:} We introduce HDMs as an elegant tool to formalize distance problems in hyperbolic space; this is analogous to Euclidean distance matrices (EDM). We derive a semidefinite characterization of HDMs by studying the properties of hyperbolic Gram matrices---matrices of Lorentzian (indefinite) inner products of points in a hyperbolic space.
\item {\bf A flexible algorithm for hyperbolic distance geometry problems (HDGPs):}  We use the semidefinite characterization to propose a flexible embedding algorithm based on semidefinite programming. It allows us to seamlessly combine metric and non-metric problems in one framework and to handle a diverse set of constraints. The non-metric and metric measurements are imputed as linear and quadratic constraints.
\item {\bf Spectral factorization and projection:} We compute the final hyperbolic embeddings with a simple, closed-form spectral factorization method.\footnote{\label{fn1}After posting the first version of our manuscript we became aware that such a one-shot spectral factorization technique was proposed at least as early as in \cite{wilson2014spherical}. The same technique is also used by \cite{keller2020hydra}.}
 We also propose a suboptimal method to find a low-rank approximation of the hyperbolic Gramian in the desired dimension. 
\end{itemize}
\begin{table}[t]
  \centering
    \caption{Essential elements in semidefinite approach for distance problems, Euclidean versus hyperbolic space.}
      	\small
\begin{tabular}{cc}
\toprule
Euclidean & Hyperbolic\\
\toprule
\rowcolor{Gray} Euclidean Distance Matrix & Hyperbolic Distance Matrix\\
Gramian & H-Gramian \\
\rowcolor{Gray} Semidefinite relaxation & Semidefinite relaxation \\
\rowcolor{Gray} to complete an EDM & to complete an HDM \\
Spectral factorization of a  & Spectral factorization of an \\
Gramian to estimate the points & H-Gramian to estimate the points\\
\bottomrule
\end{tabular}
  \label{tab:euclidean_versus_loid}
\end{table}

\subsection{Paper Organization}
We first briefly review the analytical models of hyperbolic space and formalize hyperbolic distance geometry problems (HDGPs) in \Cref{sec:HDGP}. Our framework is  parallel with semidefinite approaches for Euclidean distance problems as per \Cref{tab:euclidean_versus_loid}. In the 'Loid model, we define hyperbolic distance matrices to compactly encode hyperbolic distance measurements. We show that an HDM can be characterized in terms of the matrix of \textit{indefinite} inner products, the hyperbolic Gramian. In \Cref{sec:hdgp_hdm}, we propose a semidefinite representation of hyperbolic Gramians, and in turn HDMs. We cast HDGPs as rank-constrained semidefinite programs, which are then convexified by relaxing the rank constraints. We develop a spectral method to find a sub-optimal low-rank approximation of the hyperbolic Gramian, to the correct embedding dimension. Lastly, we propose a closed-form factorization method to estimate the embedded points. This framework lets us tackle a variety of embedding problems, as shown in \Cref{sec:experimental_results}, with real (odors) and synthetic (random trees) data. The proofs of propositions and derivations of proposed algorithms are given in the appendix and a summary of used notations is given in \Cref{tab:notations}.

\begin{table}[h]
  \centering
    \caption{Summary of notations.}
	\scriptsize
\begin{tabular}{ll}
\toprule
Symbol & Meaning\\
\toprule
\rowcolor{Gray}$[m]$  & Short for $\set{1, \ldots, m}$ \\
$[M]^2_{\mathrm{as}}$  & Asymmetric pairs $\set{(m,n): m < n, m,n \in [M]}$ \\
\rowcolor{Gray} $x = [x_0, \ldots, x_{m-1}]^\T$  & A vector in $\R^m$ \\
$X=(x_{i,j})_{i \in [m], j \in [n]}$  & A matrix in $\R^{m \times n}$ \\
\rowcolor{Gray} $X \succeq 0$ &  A positive semidefinite (square) matrix \\
 $\norm{X}_{F}$  & Frobenius norm of $X$ \\
\rowcolor{Gray} $\norm{X}_{2}$ & Operator norm of $X$ \\
 $\norm{X}_{1,2}$ & The $\ell_{2}$ norm of columns' $\ell_1$ norms, $\norm{[\norm{x_1}_{1}, \ldots, \norm{x_n}_{1}]^\T}_2$\\
\rowcolor{Gray} $\mathbb{E}_N[x]$ & Empirical expectation of a random variable, $N^{-1}\sum_{n=1}^{N}x_n$ \\
 $e_m \in \R^M$ & The $m$-th standard basis vector in $\R^M$ \\
\rowcolor{Gray}$P_{r}(X)$ & The projection of $X \succeq 0$ onto the span of its top $r$ eigenvectors \\
  $1$ & All-one vector of appropriate dimension\\
\rowcolor{Gray} $0$ &  All-zero vector of appropriate dimension \\
\bottomrule
\end{tabular}
  \label{tab:notations}
\end{table}

\section{HYPERBOLIC DISTANCE GEOMETRY PROBLEMS}\label{sec:HDGP}
\subsection{Hyperbolic Space} \label{sec:hyperbolic_space}

Hyperbolic space is a simply connected Riemannian manifold with constant negative  curvature \cite{cannon1997hyperbolic, benedetti2012lectures}. In comparison, Euclidean and elliptic geometries are spaces with zero (flat) and constant positive curvatures. There are five isometric models for hyperbolic space: half-space ($\bH^d$), Poincar\'e (interior of the disk) ($\bI^d$), jemisphere ($\bJ^d$), Klein ($\bK^d$), and 'Loid ($\bL^d$)  \cite{cannon1997hyperbolic} (\Cref{fig:hyperbola}). Each provides unique insights into the properties of hyperbolic geometry.
\begin{figure}[t]
   \includegraphics[width=.9 \linewidth]{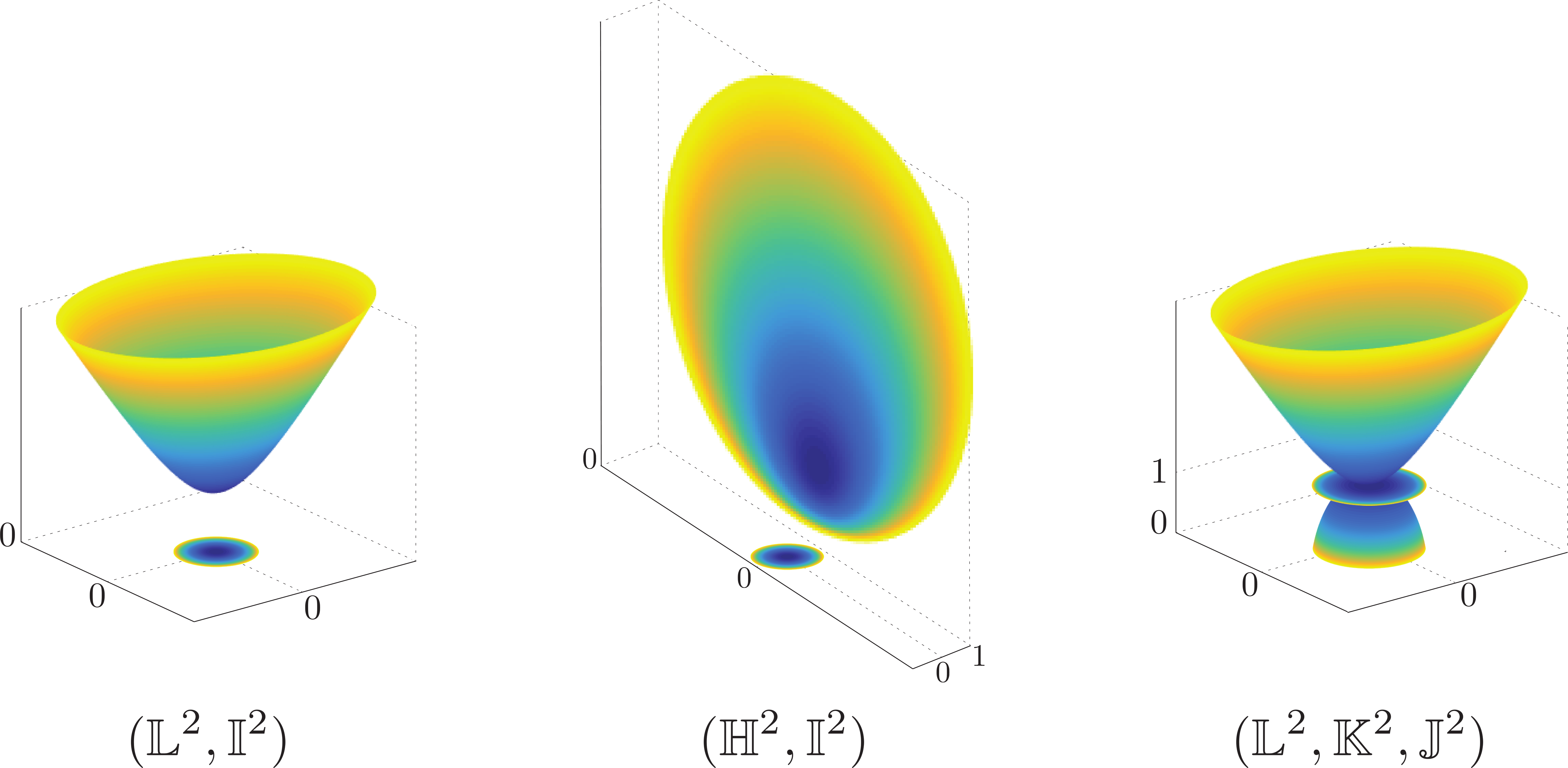}
  \caption{Models of hyperbolic space with level sets (colors) illustrating isometries.}
  \label{fig:hyperbola}
\end{figure}

In the machine learning community the most popular models of hyperbolic geometry are Poincar\'e and 'Loid. We work in the 'Loid model as it has a simple, tractable distance function. It lets us cast the HDGP (formally defined in \Cref{sub:HDGP}) as a rank-constrained semidefinite program. Importantly, it also leads to a closed-form embedding by a spectral method. For better visualization, however, we map the final embedded points to the Poincar\'e model via the stereographic projection, see \Cref{sec:poincare_model,sec:experimental_results}.
\subsubsection{'Loid Model} \label{sec:loid_model}
Let $x$ and $y$ be vectors in $\R^{d+1}$ with $d \geq 1$. The Lorentzian inner product of $x$ and $y$ is defined as
\begin{equation} \label{eq:lorentz_inner_product}
[x, y] = x^\T Hy,
\end{equation}
where 
\begin{equation}\label{eq:H}
H = \begin{pmatrix}
-1 & 0^\T\\
0 & I
\end{pmatrix} \in \R^{(d+1) \times (d+1)}.
\end{equation}
This is an indefinite inner product on $\R^{d+1}$. 
The Lorentzian inner product has almost all the properties of ordinary inner products, except that 
\begin{equation*}
\norm{x}_{H}^2 \bydef [x,x]
\end{equation*}
can be positive, zero, or negative. The vector space $\R^{d+1}$ equipped with the Lorentzian inner product \eqref{eq:lorentz_inner_product} is called a Lorentzian $(d+1)$-space, and is denoted by $\R^{1,d}$. In a Lorentzian space we can define notions similar to the Gram matrix, adjoint, and unitary matrices known from Euclidean spaces as follows.

\begin{definition}[H-adjoint \cite{gohberg1983matrices}]
The H-adjoint $R^{[*]}$ of an arbitrary matrix $R \in \R^{(d+1) \times (d+1)}$ is characterized by
\begin{equation*}
[Rx,y] = [x,R^{[*]}y], ~ \forall x, y \in \R^{d+1}.
\end{equation*}
Equivalently, 
\begin{equation}\label{eq:H_adjoint_formula}
R^{[*]} = H^{-1} R^\T H.
\end{equation}
\end{definition} 
\begin{definition}[H-unitary matrix \cite{gohberg1983matrices}]
An invertible matrix $R$ is called H-unitary if $R^{[*]} = R^{-1}$ .
\end{definition}

The 'Loid model of $d$-dimensional hyperbolic space is a Riemannian manifold $\mathcal{L}^d= (\bL^{d} , (g_x)_x )$, where
\begin{equation*}
\bL^{d} = \set{x \in \R^{d+1}: \norm{x}_{H}^2 = - 1, x_0 > 0}
\end{equation*}
and $g_x = H$ is the Riemannian metric. 

\begin{definition}[Lorentz Gramian, H-Gramian] \label{def:lorentz_gramain}
Let the columns of $X = [x_1, x_2, \cdots, x_N]$ be the positions of $N$ points in $\R^{d+1}$ (resp. $\mathbb{L}^d$). We define their corresponding Lorentz Gramian (resp. H-Gramian) as
\begin{align*}
G	 &= \left( [x_i, x_j] \right)_{i,j \in [N]} \\
&= X^\T H X
\end{align*}
where $H$ is the indefinite matrix given by \eqref{eq:H}. 
\end{definition}
The subtle difference between the Lorentz Gramian (defined for points in $\R^{d+1}$) and the H-Gramian (defined only on $\bL^d \subset \R^{d+1}$) will be important for the low-rank projection and the spectral factorization steps in Section \ref{sec:hdgp_hdm}. The indefinite inner product \eqref{eq:lorentz_inner_product} also determines the distance between $x,y \in \bL^d$, as
\begin{equation}\label{eq:loid_dist}
d(x, y) = \mathrm{acosh}(-[x, y]).
\end{equation}

\subsubsection{Poincar\'e Model} \label{sec:poincare_model}
In the Poincar\'e model ($\bI^d$), the points reside in the unit $d$-dimensional Euclidean ball. The distance between $x,y \in \bI^d$ is given by
\begin{equation}\label{eq:poincare_dist}
d(x, y) = \mathrm{acosh} \Big( 1 + 2\frac{\norm{x-y}^2}{(1-\norm{x}^{2})(1-\norm{y}^{2})} \Big).
\end{equation}
The isometry between the 'Loid and the Poincar\'e model, $h : \bL^{d} \rightarrow \bI^{d}$ is called the \textit{stereographic projection}. For $y = h(x)$, we have
\begin{equation} \label{eq:h_x}
y_i = \frac{x_{i+1}}{x_0+1},
\end{equation}
The inverse of stereographic projection is given by
\begin{equation}\label{eq:h_inv_y}
x = h^{-1} \left(  y \right) \\
= \frac{1}{1-\norm{y}^{2}} \left[ \begin{array}{c} 1+\norm{y}^{2}  \\ 2y \end{array}\right] .
\end{equation}
The isometry between the 'Loid and Poincar\'e models makes them equivalent in their embedding capabilities. However, the Poincar\'e model facilitates visualization of the embedded points in a bounded disk, whereas the 'Loid model is an unbounded space. 

\subsection{Hyperbolic Distance Problems}\label{sub:HDGP}

In a metric hyperbolic distance problem, we want to find a point set $x_1, \ldots, x_N \in \bL^d$, such that 
\begin{equation*}
d_{mn} = \mathrm{acosh} \left( -[x_m, x_n] \right), ~ \mbox{ for all } (m,n) \in \mathcal{C},
\end{equation*}
for a subset of measured distances $\mathcal{D} = \set{d_{mn}: (m,n) \in \mathcal{C} \subseteq [N]^2_{\mathrm{as}}}$.

In many applications we have access to the \textit{true} distances only through an unknown non-linear map $\tilde{d}_{mn} = \phi(d_{mn})$; examples are connectivity strength of neurons \cite{giusti2015clique} or odor co-ocurrence statistics \cite{zhou2018hyperbolic}. If all we know is that $\phi(\cdot)$ is a monotonically increasing function, then only the ordinal information has remained intact,
\begin{equation*}
d_{kl} \leq d_{mn} \Leftrightarrow \tilde{d}_{kl} \leq \tilde{d}_{mn}.
\end{equation*}
This leads to non-metric problems in which the measurements are in the form of binary comparisons \cite{agarwal2007generalized}.

\begin{definition}\label{def:ordinal_set}
For a set of binary distance comparisons of the form $d_{kl} \leq d_{mn}$, we define the set of ordinal distance measurements as
\begin{equation*}
\mathcal{O} = \set{(k,l,m,n): d_{kl} \leq d_{mn}, (k,l), (m,n) \in [N]^2_{\mathrm{as}}  }.
\end{equation*}
\end{definition}

We are now in a position to give a unified definition of metric and non-metric embedding problems in a hyperbolic space.
\begin{problem}\label{prob:dgp_general}
A hyperbolic distance geometry problem aims to find $x_1, \ldots, x_N \in \bL^d$, given 
\begin{itemize}
\item a subset of pairwise distances $\mathcal{D}$ such that
\begin{equation*}
d_{mn} = d(x_m, x_n), ~~ \mbox{ for all } d_{mn} \in \mathcal{D}
\end{equation*}
\item and/or a subset of ordinal distances measurements $\mathcal{O}$ such that 
\begin{equation*}
d(x_{i_1}, x_{i_2}) \leq d( x_{i_3}, x_{i_4}), ~~ \mbox{ for all } i \in \mathcal{O}.
\end{equation*}
\end{itemize} 
where $d(x,y) = \mathrm{acosh} \left( -[x,y] \right)$ and $i =  (i_1, i_2, i_3, i_4)$.
\end{problem}
We denote the complete sets of metric and non-metric measurements by $\mathcal{D}_c$ and $\mathcal{O}_c$.

\section{HYPERBOLIC DISTANCE MATRICES}\label{sec:hdgp_hdm}
We now introduce hyperbolic distance matrices in analogy with Euclidean distance matrices to compactly encode inter-point distances of a set of points $x_1, \ldots, x_N \in \bL^d$.
\begin{definition}
The hyperbolic distance matrix (HDM) corresponding to the list of points $X = [x_1, \ldots, x_N] \in (\bL^{d} )^N$ is defined as
\begin{equation*}
D = \mathcal{D}(X) =\left( d(x_i, x_j) \right)_{i,j \in [N]}.
\end{equation*}
The $ij$-th element of $\mathcal{D}(X)$ is hyperbolic distance between $x_i$ and $x_j$, given by $d(x_i, x_j) = \mathrm{acosh}(-[x_i , x_j])$ and for all $i,j \in [N]$.
\end{definition}

HDMs are characterized by Lorentzian inner products which allows us to leverage the definition of an H-Gramian (\Cref{def:lorentz_gramain}).
Given points $x_1, \ldots, x_N \in \bL^d$, we compactly write the HDM corresponding to $G$ as
\begin{equation}
    \label{eq:hdm_via_gram}
    D  = \mathrm{acosh} [-G],
\end{equation}
where $\mathrm{acosh} [\cdot]$ is an elementwise $\mathrm{acosh}(\cdot)$ operator.  

We now state our first main result: a  semidefinite characterization of $H$-Gramians. This is a key step in casting HDGPs as rank-constrained semidefinite programs. 
\begin{restatable}[Semidefinite characterization of H-Gramian]{proposition}{chg}\label{prop:char_hyper_gram}
Let $G$ be the hyperbolic Gram matrix for a set of points $x_1, \cdots, x_N \in \bL^{d }$. Then, 
\begin{align*} 
&         G = G^{+} - G^{-}   & & & \\
& \text{where}     && G^{+}, G^{-} \succeq 0, & \\
&                   	& & \rank{G^{+}} \leq d, \\
&  					& & \rank{G^{-}} \leq 1, \\
&                   	& & \diag{G} = -1, \\
&                   	& &e_i^\T G e_j \leq -1, ~~ \forall i,j \in [N].
\end{align*}
Conversely, any matrix $G \in \R^{N \times N}$ that satisfies the above conditions is a hyperbolic Gramian for a set of $N$ points in $\bL^d$.
\end{restatable}
The proof is given in \Cref{prop:char_hyper_gram}.

\subsection{Solving for the H-Gramians} \label{sec:semi_def_hdgp}

While \Cref{prob:dgp_general} could be formalized directly in $X$ domain, this approach is unfavorable as the optimization domain, $\bL^d$, is a non-convex set. What is more, the hyperbolic distances
\begin{equation}\label{eq:non_linear_term}
d(x_m, x_n) =  \mathrm{acosh}\left( - e_m^{\T} X^{\T} H Xe_n \right)
\end{equation}
are non-linear functions of $X$ with an unbounded gradient \cite{de2018representation}.
Similar issues arise when computing embeddings in other spaces such as Euclidean \cite{dokmanic2015euclidean} or the space of polynomial trajectories \cite{tabaghi2019kinetic}. A particularly effective strategy in the Euclidean case is the semidefinite relaxation which relies on the simple fact that the Euclidean Gramian is positive semidefinite. We thus proceed by formulating a semidefinite relaxation for hyperbolic embeddings based on \Cref{prop:char_hyper_gram}. 

Solving the HDGP involves two steps, summarized in \Cref{alg:hdgp}:
\begin{enumerate}
\item Complete and denoise the HDM via a semidefinite program;
\item Compute an embedding of the clean HDM: we propose a closed-form spectral factorization method.
\end{enumerate}
Note that step (2) is independent of step (1): given accurate hyperbolic distances, spectral factorization will give the points that reproduce them. However, since the semidefinite relaxation might give a Gramian with a higher rank than desired, eigenvalue thresholding in step (2) might move the points off of $\bL^d$. That is because eigenvalue thresholding can violate the necessary condition for the hyperbolic norm, $\norm{x}_H^2 = -1$, or $\diag{G} = -1$ in \Cref{prop:char_hyper_gram}. We fix this by projecting each individual point to $\bL^d$. The spectral factorization and the projection are summarized in \Cref{alg:sdp,alg:RP}.

\begin{algorithm}[b]
\caption{HDGP algorithm}\label{alg:hdgp}
\begin{algorithmic}[1]
\Procedure{$\mathtt{HDGP}$}{$\wt{D}, \wt{\mathcal{O}}, d$}
\\ \hspace*{\algorithmicindent} \textbf{Input}: Incomplete and noisy distance matrix, $\wt{D}$, and ordinal measurements, $\wt{\mathcal{O}}$, and embedding dimension, $d$.
\State $G = \mathtt{SDR}(\wt{D}, \wt{\mathcal{O}}, d)$ \Comment{Complete \& denoise HDM.}
\State $X = \mathtt{Embed}(G,d)$ \Comment{Embed points in $\bL^d$.}
\State For $X = [x_1, \ldots, x_N] \in \left(\bL^{d} \right)^N$, let 
\begin{equation*}
y_n = h(x_n), ~ \forall n \in [N]
\end{equation*}
where $h(\cdot)$ is given by \eqref{eq:h_x}.  \Comment{Map the points to $\bI^d$.}
\State \textbf{return} $Y = [y_1, \cdots, y_N] \in \left(\bI^{d} \right)^N$.
\EndProcedure
\end{algorithmic}
\end{algorithm}

Let $\wt{D}$ be the measured noisy and incomplete HDM, with unknown entries replaced by zeroes. We define the mask matrix $W = (w_{ij})$ as
\begin{equation*}
    w_{ij} \bydef
    \begin{cases}
       1, \ \text{for} ~ (i,j) \in \mathcal{C}  \vee (j,i) \in \mathcal{C}\\
       0,\ \text{otherwise.}
    \end{cases}.
\end{equation*}
This mask matrix lets us compute the loss only at those entries that were actually measured. We use the semidefinite characterization of hyperbolic Gramians in \Cref{prop:char_hyper_gram} to complete and denoise the measured HDM, and eventually solve HDGP. 

Although the set of hyperbolic Gramians for a given embedding dimension is non-convex due to the rank constraints, discarding the rank constraints results in a straightforward semidefinite relaxation.
\begin{algorithm}[b]
\caption{Semidefinite relaxation for HDGP}\label{alg:sdp}
\begin{algorithmic}[1]
\Procedure{$\mathtt{SDR}$}{$\wt{D}, \wt{\mathcal{O}},d$}
\\ \hspace*{\algorithmicindent} \textbf{Input}: Incomplete and noisy distance matrix, $\wt{D}$, and ordinal measurements, $\wt{\mathcal{O}}$, and embedding dimension, $d$.
\State Let $W$ be the measurement mask.
\State For small $\epsilon_1, \epsilon_2 > 0$, solve for $G$:
\begin{align*}
& \text{minimize}       & &  \mathrm{Tr} \ G^{+} + \mathrm{Tr} \ G^{-}  &  \\
& \text{w.r.t}       & &  G^{+}, G^{-}  \succeq 0 &  \\
& \text{subject to}   &&  G = G^{+}-G^{-},  \\
&                   	& & \diag{G} = -1, & \\
&                   	& & e_i^\T G e_j \leq -1, & \forall i,j \in [N]\\
&                   	& & \norm{ W \circ \big( \mathrm{cosh}[\wt{D}] + G\big) }_{F}^{2} \leq  \epsilon_1, \\
&                   	& & \mathcal{L}_k(G) \geq \epsilon_2, & \forall k \in \wt{\mathcal{O}}. 
\end{align*}
\State \textbf{return} $G$.
\EndProcedure
\end{algorithmic}
\end{algorithm}
\begin{table*}[t]
\centering
\caption{Examples of specialized HDGP objectives.}
\begin{tabular}[t]{lll}
\hline
\hfil Cost function & \hfil Parameters & \hfil Applications \\
\hline
\rowcolor{Gray} & $\sbullet[.75]$ $W_k^{+} =I$, $ W_k^{-} = I$ &   \\
\rowcolor{Gray}
 & $\sbullet[.75]$ $W_k^{+} =(G_k^{+} + \delta I)^{-1}$, $ W_k^{-} = (G_k^{-} + \delta I)^{-1}$ &  \\
\rowcolor{Gray} \multirow{-3}{*}{$\mathrm{Tr} \ W_k^{+}G^{+} + \mathrm{Tr} \ W_k^{-}G^{-}$}   & $\sbullet[.75]$ $W_k^{+} = I-P_{d}(G_k^{+})$, $ W_k^{-} = I-P_{1}(G_k^{-})$ & \multirow{-3}{*}{Low-rank hyperbolic embedding \cite{jawanpuria2019low,fazel2002matrix,fornasier2011low,fazel2003log}} \\
\multirow{2}{*}{$\mathrm{Tr} \ G^{+} + \mathrm{Tr} \ G^{-}+\sum_{k}p_k\epsilon_k$}  &  $\sbullet[.75]$ $p_k = 1$ &  Ordinal outlier removal \cite{olsson2010outlier,seo2009outlier,yu2014adversarial}, \\
&$\sbullet[.75]$  $\sum_{k}p_k = M, 0 \leq p_k \leq 1$ & Robust hierarchical embedding \cite{nickel2017poincare,ma2019robust} 
\\
\rowcolor{Gray} $\mathrm{Tr} \ G^{+} + \mathrm{Tr} \ G^{-}+\lambda \norm{C}_{1,2}$ &  $\sbullet[.75]$ $\norm{ \mathrm{cosh}[\wt{D}] + G^{+}-G^{-}+C }_{F} \leq \epsilon$ & Anomaly detection in  weighted graphs \cite{akoglu2010oddball} and networks \cite{yu2018netwalk}
\\
\hline
\end{tabular}  \label{tab:HDGP_setups}
\end{table*}%
However, if we convexify the problem by simply discarding the rank constraints, then all pairs $(G_1, G_2) \in \set{(G^{+}+P, G^{-}+P): P \succeq 0}$ become a valid solution. On the other hand, since 
\begin{equation*}
\mathrm{rank} \ G + P \geq \mathrm{rank} \ G ~ \mbox{ for } ~ G,P \succeq 0,
\end{equation*}
we can eliminate this ambiguity by promoting low-rank solutions for $G^{+}$ and $ G^{-}$. While directly minimizing
\begin{equation}\label{eq:low_rank_objective}
\mathrm{rank} \ G^{+} + \mathrm{rank} \ G^{-}
\end{equation}
is NP-hard \cite{vandenberghe1996semidefinite}, there exist many approaches to make  \eqref{eq:low_rank_objective} computationally tractable, such as trace norm minimization \cite{mishra2013low}, iteratively reweighted least squares minimization \cite{fornasier2011low}, or the log-det heuristic \cite{fazel2003log} that minimizes the following smooth surrogate for \eqref{eq:low_rank_objective}:
\begin{equation*}
\log \det (G^{+}+\delta I) + \log \det ( G^{-} + \delta I),
\end{equation*}
where $\delta > 0$ is a small regularization constant. This objective function is linearized as $C + \mathrm{Tr} \ W_k^{+}G^{+} + \mathrm{Tr} \ W_k^{-} G^{-}$ for $W_k^{+} = (G_k^{+} + \delta I)^{-1}$ and $W_k^{-} = (G_k^{-} + \delta I)^{-1}$, which can be iteratively minimized\footnote{In practice, we choose a diminishing sequence of $\delta_k$.}. In our numerical experiments we will uset he trace norm minimization unless otherwise stated. Then, we enforce the data fidelity objectives and the properties of the embeddings space (Proposition \ref{prop:char_hyper_gram}) in the form of a variety of constraints.

{\bf Metric embedding:} The quadratic constraint 
\begin{equation*}
\norm{ W \circ \big( \mathrm{cosh}[\wt{D}] + G\big) }_{F}^{2} \leq  \epsilon_1
\end{equation*}
makes sure the hyperbolic Gramian, $G$, accurately reproduces the given distance data.

{\bf Non-metric embedding:} The ordinal measurement constraint of
\begin{equation*}
d(x_{i_1}, x_{i_2}) \leq d(x_{i_3}, x_{i_4}), 
\end{equation*}
is simply a linear constraint in form of
\begin{equation*}
\mathcal{L}_i(G) =e_{i_1}^{\T} G e_{i_2} - e_{i_3}^{\T} G  e_{i_4}\geq 0
\end{equation*}
where $i \in \mathcal{O}$ and $i = (i_1, i_2, i_3, i_4)$. In practice, we replace this constraint by $\mathcal{L}_i(G) \geq \epsilon_2 > 0$ to avoid trivial solutions.

{\bf 'Loid model:} The unit hyperbolic norm appears as a simple linear constraint
\begin{equation*}
\diag{G} = -1,
\end{equation*}
which guarantees that the embedded points reside in sheets $\bL^d \cup -\bL^d$. Finally, $e_i^\T G e_j \leq -1$ enforces all embedded points to belong to the same hyperbolic sheet, i.e. $x_n \in \bL^d$ for all $n \in [N]$.

This framework can serve as a bedrock for multitude of other data fidelity objectives. We can seamlessly incorporate \textit{outlier removal} schemes by  introducing slack variables into the objective function and constraints \cite{olsson2010outlier,seo2009outlier,yu2014adversarial}. For example, the modified objective function
\begin{equation*}
\mathrm{Tr} \ G^{+} + \mathrm{Tr} \ G^{-}+\sum_{k}\epsilon_k
\end{equation*}
can be minimized subject to $\mathcal{L}_k(G) +\epsilon_k \geq 0$ and $\epsilon_k \geq 0$ as a means of removing outlier comparisons (we allow some comparisons to be violated; see \Cref{sec:odor_embedding} for an example).

We can similarly implement outlier detection in metric embedding problems. As an example, we can adapt the outlier pursuit algorithm \cite{xu2010robust}. Consider the measured $H$-Gramian of a point set with a few outliers
\begin{equation*}
\hat{G} = G+ C + N
\end{equation*}
where $G$ is outlier-free hyperbolic Gramian, $C$ is a matrix with only few non-zero columns and $N$ represents the measurement noise. Outlier pursuit aims to minimize a convex surrogate for 
\begin{equation*}
\rank{G} + \lambda \norm{C}_{0,c} ~ \mbox{s.t.} ~ \norm{\hat{G} - G - C}_{F}^2 \leq \epsilon
\end{equation*}
where $\norm{C}_{0,c}$ is the number of non-zero columns of $C$; more details and options are given in \Cref{tab:HDGP_setups}. 

We note that scalability of semidefinite programs has been studied in a number of recent works \cite{majumdar2019recent}, for example based on sketching \cite{yurtsever2019scalable,yurtsever2017sketchy}. 

\subsection{Low-rank Approximation of H-Gramians}
From \Cref{prop:char_hyper_gram}, it is clear that the rank of a hyperbolic Gramian of points in $\bL^d$ is at most $d+1$. However, the H-Gramian estimated by the semidefinite relaxation in \Cref{alg:sdp} does not necessarily have the correct rank.
Therefore, we want to find its best rank-$(d+1)$ approximation, namely $\hat{G}$, such that
\begin{equation}\label{eq:low_rank_hyperbolic_approx}
\norm{G - \hat{G}}_{F}^2 = \inf_{X \in {\left( \bL^{d} \right)}^N} \norm{G - X^\T H X}_{F}^2.
\end{equation}
In Algorithm \Cref{alg:RP} we propose a simple but suboptimal procedure to solve this low-rank approximation problem. Unlike iterative refinement algorithms based on optimization on manifolds \cite{jawanpuria2019low}, our proposed method is one-shot. It is based on the spectral factorization of the the estimated hyperbolic Gramian and involves the following steps:
\begin{itemize}
\item {\bf Step 1:} We find a set of points $\set{z_n}$ in $\R^{d+1}$, whose Lorentz Gramian best approximates $G$; See \Cref{def:lorentz_gramain} and lines $2$ to $5$ of \Cref{alg:RP}. In other words, we relax the optimization domain of \eqref{eq:low_rank_hyperbolic_approx} from $\bL^d$ to $\R^{d+1}$, 
\begin{equation*}
Z = \argmin_{X \in \R^{(d+1)\times N}}\norm{G - X^\T H X}^2.
\end{equation*}
\item {\bf Step 2:} We project each point $z_n$ onto $\bL^d$, i.e.
\begin{equation*}
\hat{X} = \argmin_{X \in \left(\bL^{d}\right)^N} \norm{X - Z}_{F}^{2}.
\end{equation*}
This gives us an approximate rank-$(d+1)$ hyperbolic Gramian, $\hat{G} = \hat{X}^\T H \hat{X}$; see \Cref{fig:tangent} and \Cref{sec:projection}.
\end{itemize}

\begin{algorithm}[b]
\caption{Low-rank approximation and spectral factorization of hyperbolic Gramian}\label{alg:RP}
\begin{algorithmic}[1]
\Procedure{$\mathtt{Embed}$}{$G, d$}
\\ \hspace*{\algorithmicindent} \textbf{Input}: Hyperbolic Gramian $G$, and embedding dimension $d$.
\State Let $U^\T \Lambda U$ be eigenvalue decomposition of $G$, where $\Lambda = \diag{(\lambda_0, \cdots, \lambda_{N-1})}$ such that 
\begin{itemize}
\item $\lambda_0 = \min_{i} \lambda_i$,
\item $\lambda_i$ is the top $i$-th element of $\set{\lambda_i}$ for $i \in [N]-1$.
\end{itemize}
\State Let $G_{d+1} = U_d^\T \Lambda_d U_d$, where 
\begin{equation*}
\Lambda_d = \diag{\big(\lambda_0, u(\lambda_1), \cdots, u(\lambda_d) \big)},
\end{equation*}
$u(x) = \max \set{x, 0}$, and $U_d$ be the corresponding sliced eigenvalue matrix.
\State  $Z = R|\Lambda_d|^{1/2} U_d^\T$, for any H-unitary matrix $R$.
\State For $Z = [z_1, \ldots, z_N]$, let 
\begin{equation*}
x_n = \mathtt{Project}(z_n), ~ \forall n \in [N]
\end{equation*}
\State \textbf{return} $X = [x_1, \ldots, x_N] \in \left(\bL^{d}\right)^N$.
\EndProcedure
\end{algorithmic}
\end{algorithm}

The first step of low-rank approximation of a hyperbolic Gramian $G$ can be interpreted as finding the positions of points in $\R^{d+1}$ (not necessarily on $\bL^d$) whose Lorentz Gramian best approximates $G$. 

\subsection{Spectral Factorization of H-Gramians} \label{sec:specfac}

To finally compute the point locations, we describe a spectral factorization method, proposed in \cite{wilson2014spherical} (cf. footnote \ref{fn1}), 
, to estimate point positions from their Lorentz Gramian (line $5$ of \Cref{alg:RP}). This method exploits the fact that Lorentz Gramians have only one non-positive eigenvalue (see \Cref{lem:lorentz_gram_char} in the appendix) as detailed in the following proposition.
\begin{proposition}\label{prop:spectral_factorization}
Let $G$ be a hyperbolic Gramian for $X \in \left( \bL^d \right)^N$, with eigenvalue decomposition $G = U \Lambda U^\T$, and eigenvalues $\lambda_0 \leq 0 \leq \lambda_1 \leq \ldots \leq \lambda_d$.\footnote{An H-Gramian is a Lorentz Gramian.} Then, there exists an $H$-unitary matrix $R$ such that $X = R |\Lambda|^{1/2} U$.
\end{proposition}
The proof is given in \Cref{sec:proof_p2}. Note that regardless of the choice of $R$, $X = R |\Lambda|^{1/2} U$ will reproduce $G$ and thus the corresponding distances. This is the rigid motion ambiguity familiar from the Euclidean case \cite{cannon1997hyperbolic}. If we start with an $H$-Gramian with a wrong rank, we need to follow the spectral factorization by Step 2 where we project each point $z_n \in \R^{d+1}$ onto $\bL^d$. This heuristic is suboptimal, but it is nevertheless appealing since it only requires a single one-shot calculation as detailed in \Cref{sec:projection}.

\begin{figure}[t]
\center
\includegraphics[width=.6 \linewidth]{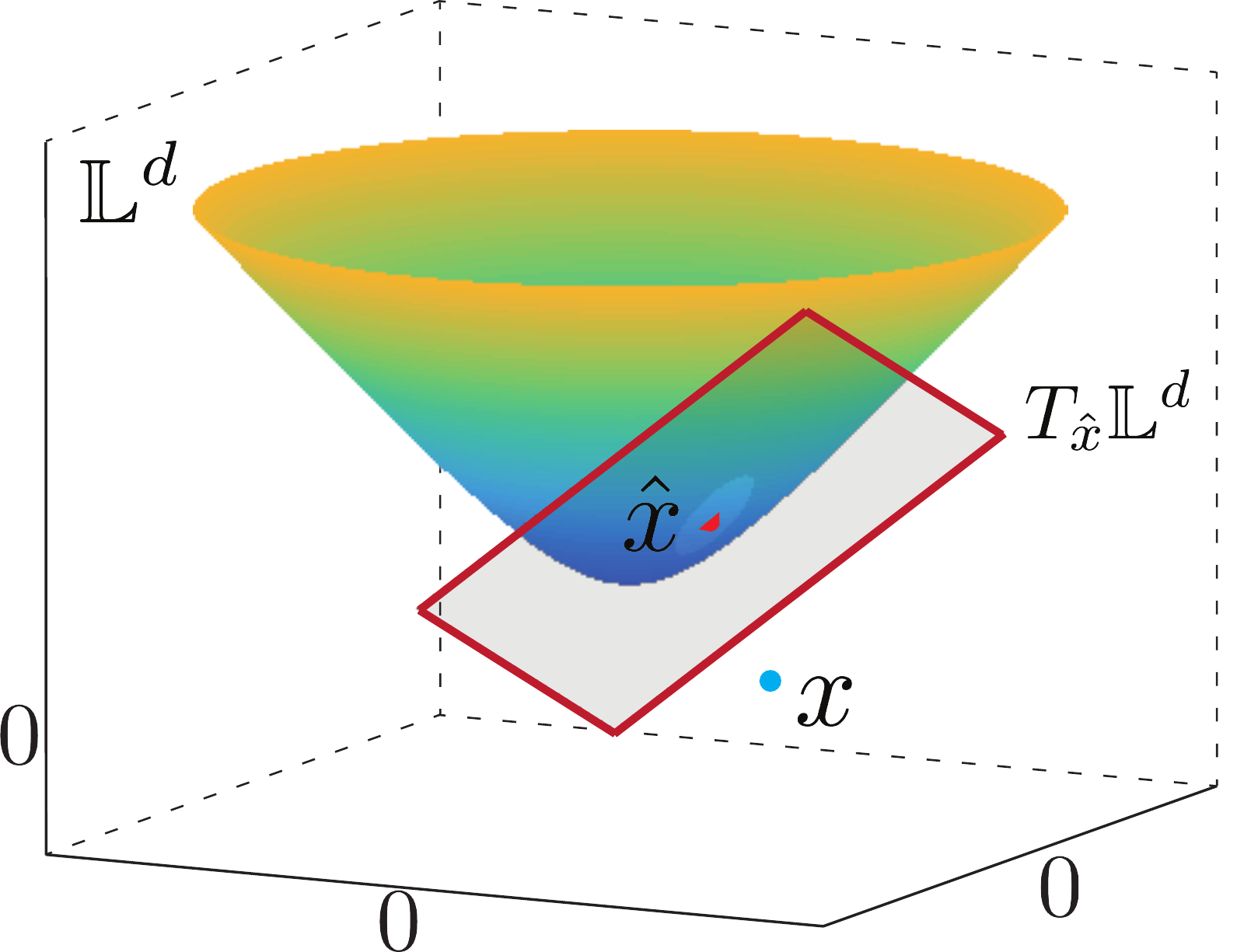}
\caption{Projecting a point  in $\R^{d+1}$ (blue) to $\bL^d$ (red).}
\label{fig:tangent}
\end{figure}

\section{Experimental Results} \label{sec:experimental_results}
In this section we numerically demonstrate different properties of \Cref{alg:hdgp} in solving HDGPs. In a general hyperbolic embedding problem, we have a mix of metric and non-metric distance measurements which can be noisy and incomplete. 
Code, data and documentation to reproduce the experimental results are available at \url{https://github.com/puoya/hyperbolic-distance-matrices}.

\subsection{Missing Measurements} \label{sec:sparse_measurements}
Missing measurements are a common problem in hyperbolic embeddings of concept hierarchies. For example, hyperbolic embeddings of words based on Hearst-like patterns rely on co-occurrence probabilities of word pairs in a corpus such as WordNet \cite{miller1998wordnet}. These patterns are sparse since word pairs must be detected in the right configuration \cite{le2019inferring}. In perceptual embedding problems, we ask individuals to rate pairwise similarities for a set of objects. It may be difficult to collect and embed all pairwise comparisons in applications with large number of objects \cite{agarwal2007generalized}.
\begin{figure*}[t]
\center
 \includegraphics[width=.77 \linewidth]{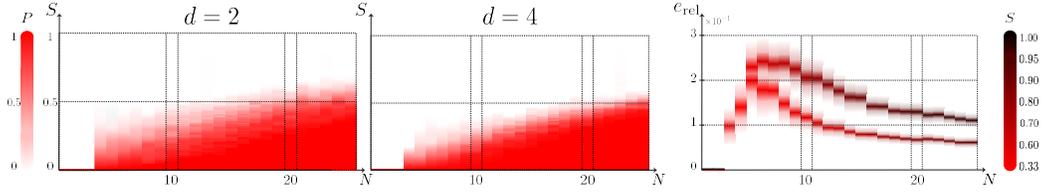}
  \caption{Left and middle: The probability of $\delta$-accurate estimation for metric sampling density $S$, $M=100$, and $\delta = 10^{-2}$. Right: The empirical error $e_{\mathrm{rel}} =\mathbb{E}_{K}[e_{\mathrm{rel}}(X)]$ for ordinal sampling density $S$, $d=2$, $M=50$, and $K = 10$. In each bar, shading width represents the empirical standard deviation of $e_{\mathrm{rel}}(X)$.}
  \label{fig:sparsity}
\end{figure*}

The proposed semidefinite relaxation gives a simple way to handle missing measurements. The \textit{metric sampling density} $0 \leq S \leq 1$ of a measured HDM is the ratio of the number of missing measurements to total number of pairwise distances, $S = 1 - \frac{|\mathcal{D}|}{|\mathcal{D}_c|}$.
We want to find the probability $p(S)$ of successful estimation given a sampling density $S$. In practice, we fix the embedding dimension, $d$, and the number of points, $N$, and randomly generate a point set, $X \in \left(\bL^{d}\right)^N$. A trial is successful if we can solve the HDGP for noise-free measurements and a random mask $W$ of a fixed size so that the estimated hyperbolic Gramian has a small relative error, 
\(
e_{\mathrm{rel}}(\hat{G}) = \frac{\norm{\mathcal{D}(X) - \mathrm{acosh}[-\hat{G}]}_{F}}{\norm{\mathcal{D}(X)}_{F}} \leq \delta.
\)
We repeat for $M$ trials, and empirically estimate the success probability as $\hat{p}(S) = \frac{M_s}{M}$ where $M_s$ is the number of successful trials. We repeat the experiment for different values of $N$ and $d$, see  \Cref{fig:sparsity}. 

For non-metric embedding applications, we want to have \textit{consistent} embedding for missing ordinal measurements. The \textit{ordinal sampling density} $0 \leq S \leq 1$ of a randomly selected set of ordinal measurements is defined as $S = 1 - \frac{|\mathcal{O}|}{|\mathcal{O}_c|}$. For a point set $X \in \left( \bL^{d} \right)^N$, we define the average relative error of estimated HDMs as
\(
e_{\mathrm{rel}}(X) = \mathbb{E}_{M}\frac{\norm{D_{\mathcal{O}}-\mathbb{E}_{M}[D_{\mathcal{O}}]}_F}{ \norm{\mathbb{E}_{M}[D_{\mathcal{O}}]}_F}
\)
where $D_{\mathcal{O}}$ is the estimated HDM for ordinal measurements $\mathcal{O}$, and empirical expectation is with respect to the random ordinal set $\mathcal{O}$. We repeat the experiment for $K$ different realizations of $X \in \left(\bL^{d}\right)^N$ (\Cref{fig:sparsity}). We can observe that across different embedding dimensions, the maximum allowed fraction of missing measurements for a consistent and accurate estimation increases with the number of points. 

\subsection{Weighted Tree Embedding} \label{sec:tree_embedding}
Tree-like hierarchical data occurs commonly in natural scenarios. In this section, we want to compare the embedding quality of weighted trees in hyperbolic and the baseline in Euclidean space. 

We generate a random tree $T$ with $N$ nodes, maximum degree of $\Delta(T) = 3$, and i.i.d. edge weights from $\mathrm{unif}(0,1)$\footnote{The most likely maximum degree for trees with $N \leq 25$ \cite{moon1968maximum}.}. Let $D_T$ be the distance matrix for $T$, where the distance between each two nodes is defined as the weight of the path joining them.

For the hyperbolic embedding, we apply \Cref{alg:sdp} with log-det heuristic objective function to acquire a low-rank embedding. On the other hand, Euclidean embedding of $T$ is the solution to the following semidefinite relaxation
\begin{align} \label{eq:sdp_euclidean}
& \text{minimize}       & &  \norm{ D_T^{\circ 2} - \mathcal{K}(G) }_{F}^{2}  &  \\
& \text{w.r.t}       & &   G  \succeq 0 \nonumber &  \\
& \text{subject to}   &&  G1 = 0\nonumber
\end{align}
where $\mathcal{K}(G) = -2G + \mathrm{diag}(G)1^\T + 1\mathrm{diag}(G)^\T$, and $D_T^{\circ 2}$ is the entrywise square of $D_T$. This semidefinite relaxation (SDR) yields a \textit{minimum error} embedding of $T$, since the embedded points can reside in an arbitrary dimensional Euclidean space.
\begin{figure}[b] \label{fig:metric_tree_embedding}
\center
 \includegraphics[width=.80 \linewidth]{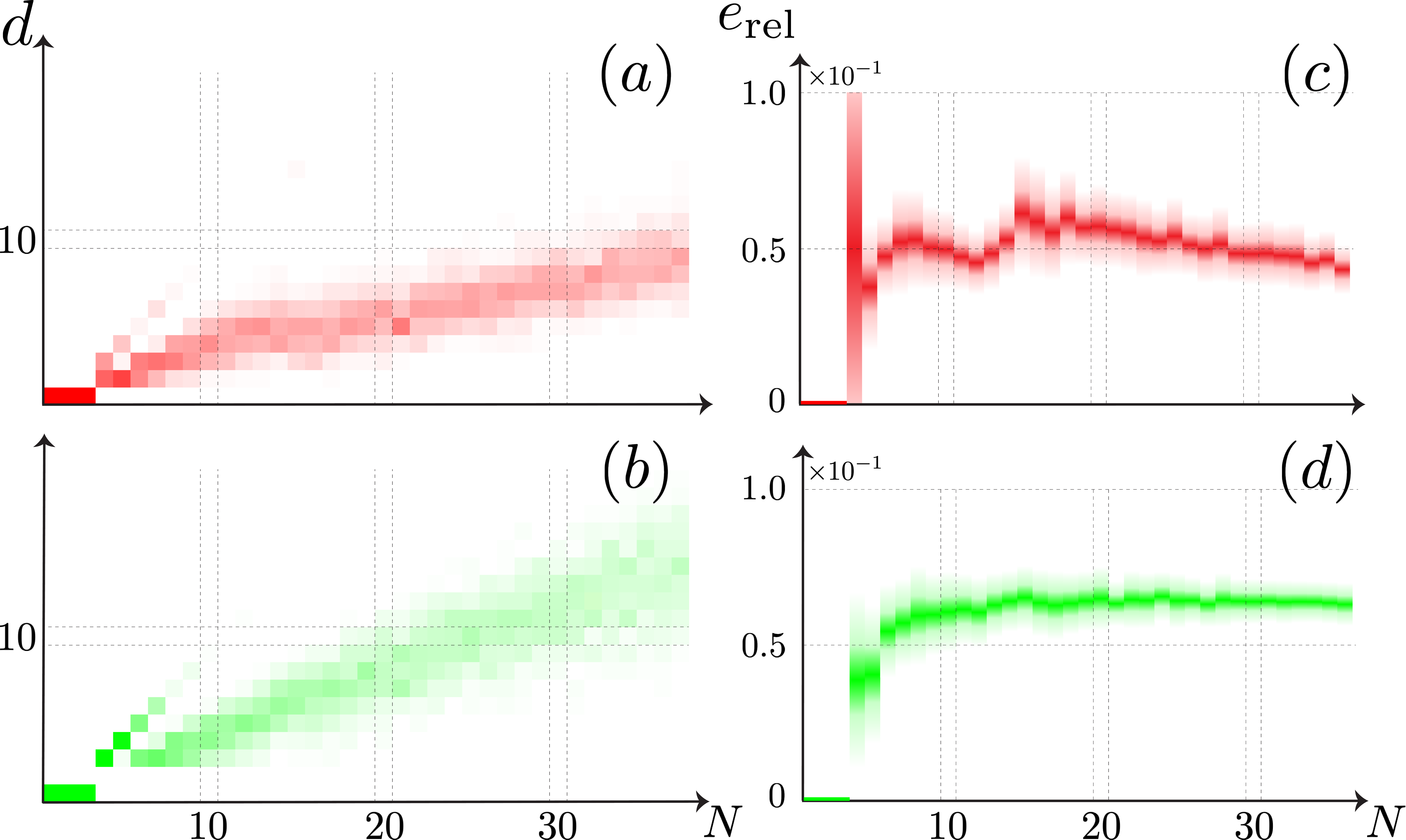}
  \caption{Tree embedding in hyperbolic (red) and Euclidean (green) space. Discrete distribution of optimal embedding dimension for $M=100$, $(a)$ and $(b)$. Average, $\mathbb{E}_{M} [e_{\mathrm{rel}}(T)]$, and standard deviation of embedding error, $(c)$ and $(d)$.}
  \label{fig:metric_tree_embedding}
\end{figure}

The embedding methods based on semidefinite relaxation are generally accompanied by a projection step to account for the potentially incorrect embedding dimension. For hyperbolic embedding problems, this step is summarized in \Cref{alg:RP}, whereas it is simply a singular value thresholding of the Gramian for Euclidean problems. Note that the SDRs always find a $(N-1)$-dimensional embedding for a set of $N$ points; see \Cref{alg:sdp} and \eqref{eq:sdp_euclidean}. In this experiment, we define the optimal embedding dimension as
\begin{equation*}
d_0 = \min \set{d \in \mathbb{N}: \frac{\norm{D_{N-1}-D_{d}}_{F}}{\norm{D_{N-1}-D_{d+1}}_{F}} \geq 1-\delta}
\end{equation*}
where $D_n$ is the distance matrix for embedded points in $\bL^n$ (or $\R^n$), and $\delta = 10^{-3}$. This way, we accurately represent the estimated distance matrix in a low dimensional space.
Finally, we define the relative (or normalized) error of embedding $T$ in $d_0$-dimensional space as
\(
e_{\mathrm{rel}}(T) = \frac{\norm{D_{T}-D_{d_0}}_{F}}{\norm{D_{T}}_{F}}.
\)
We repeat the experiment for $M$ randomly generated trees $T$ with a varying number of vertices $N$. The hyperbolic embedding yields smaller average relative error $\mathbb{E}_{M}[e_{\mathrm{rel}}(T)]$ compared to Euclidean embedding, see \Cref{fig:metric_tree_embedding}. It should also noted that the hyperbolic embedding has a lower optimal embedding dimension, even though the low-rank hyperbolic Gramian approximation is sub-optimal.

\subsection{Odor Embedding} \label{sec:odor_embedding}

In this section, we want to compare hyperbolic and Euclidean non-metric embeddings of olfactory data following the work of Zhou et al. \cite{zhou2018hyperbolic}. We conduct identical experiments in each space, and compare embedding quality of points from \Cref{alg:sdp} in hyperbolic space to its semidefinite relaxation counterpart in Euclidean space, namely generalized non-metric MDS  \cite{agarwal2007generalized}.

We use an olfactory dataset comprising mono-molecular odor concentrations measured from blueberries  \cite{gilbert2015identifying}. In this dataset, there are $N = 52$ odors across the total of $M = 164$ fruit samples.
\begin{figure}[t] 
\center
   \includegraphics[width=.73 \linewidth]{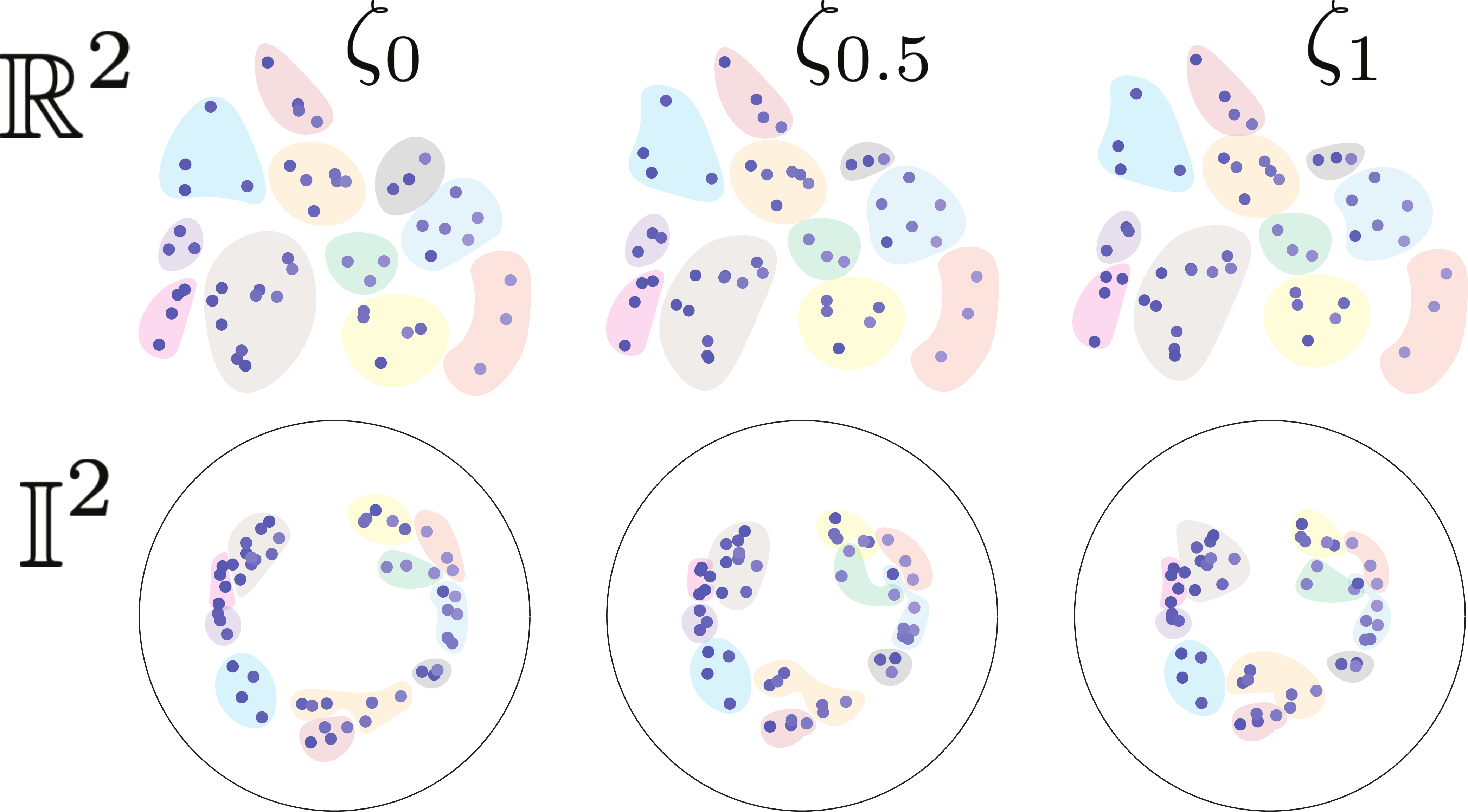}
  \caption{Embedding of odors for different levels of allowable violated measurements $\zeta_p$. Clusters with the matching colors contain the same odors.} \label{fig:poincare_odor}
\end{figure}
Like Zhou \textit{et al.} \cite{zhou2018hyperbolic}, we begin by computing correlations between odor concentrations across samples \cite{zhou2018hyperbolic}. The correlation coefficient between two odors $x_i$ and $x_j$ is defined as 
\begin{equation*}
C(i,j) = \frac{(\vx_i-\mu_{x_i}1)^\T (\vx_j-\mu_{x_j}1)}{ \norm{\vx_i -\mu_{x_i}1} \norm{\vx_j -\mu_{x_j}1}}
\end{equation*}
where $\vx_n = (x_n^{(1)}, \ldots, x_n^{(M)})^\T$, $x_i^{(m)}$ is the concentration of $i$-th odor in $m$-th fruit sample, $M$ is total number of fruit samples and $\mu_{x_n} = \frac{1}{M}\sum_{m=1}^{M} x_n^{(m)}$. 

The goal is to find an embedding for odors $y_1, \ldots, y_N  \in \bI^d$ (or $\R^d$) such that
\begin{equation*}
d(y_{i_1}, y_{i_2}) \leq d( y_{i_3}, y_{i_4}), ~~ (i_1, i_2, i_3,i_4) \in \mathcal{O},
\end{equation*}
where,
\begin{equation*}
\mathcal{O} \subseteq \mathcal{O}_c = \set{(i_1,i_2,i_3,i_4) \in \left([N]_{\mathrm{as}}^2\right)^2: C(i_1,i_2) \geq C(i_3,i_4)}.
\end{equation*}
The total number of distinct comparisons grows rapidly with the number of points, namely $|\mathcal{O}_c| =  0.87$ million. In this experiment, we choose a random set of size $|\mathcal{O}| = 2K {N \choose 2}$ for $K=4$ to have the sampling density of $S = 98.79\%$\footnote{In hyperbolic embedding, this is the ratio of number of ordinal measurements to number of variables, i.e. $K = \frac{|\mathcal{O}|}{ 2{N \choose 2} }$.}, which brings the size of ordinal measurements to $|\mathcal{O}| \approx 10^4$.

We ensure the embedded points do not collapse by imposing the following minimum distance constraint $d(x_i, x_j) \geq 1$ for all $(i, j) \in [N]^{2}_{\mathrm{as}}$; this corresponds to a simple linear constraint in the proposed formulation. An ideal order embedding accurately reconstructs the missing comparisons. We calculate the percentage of correctly reconstructed distance comparisons as $\gamma_{d} = |\widehat{\mathcal{O}}_{c,d} \cap \mathcal{O}_c| / |\mathcal{O}_c| $, where $\widehat{\mathcal{O}}_{c,d}$ is the complete ordinal set corresponding to a $d$-dimensional embedding.

A simple regularization technique helps to remove outlier measurements and improve the generalized accuracy of embedding algorithms. We introduce the parameter $\zeta_p$ to permit SDR algorithms to dismiss at most $p$-percent of measurements, namely
\begin{equation*}
\mathcal{L}_{k}(G) + \epsilon_k \geq \epsilon_2~ \mbox{and} ~ \epsilon_k \geq 0, ~ \forall k \in \mathcal{O} ~ \mbox{and} ~ \sum_{k} \epsilon_k \leq \zeta_p 
\end{equation*}
where $\zeta_p= \frac{p}{100} |O| \epsilon_2$. 

\begin{table}[t]
\caption{Reconstruction accuracy of ordinal measurements $\gamma_d$ for different levels of allowable violation $\zeta_p$.} 
\centering
\begin{tabular}{l c  ccccc} \hline
 Space &     & $d=2$&$d=4$& $d=6$& $d=8$&$d=10$ \\ \hline
\rowcolor{Gray} \cellcolor{white} & \cellcolor{white} $\zeta_0$ &$76.06$&$83.60$&$86.87$&$89.48$&$91.03$ \\
 Hyperbolic & $\zeta_{0.5}$& $76.52$&$83.71$&$86.94$&$89.68$&$91.16$\\  \rowcolor{Gray}  \cellcolor{white} & \cellcolor{white}  $\zeta_1$ & $76.43$ & $83.71$ & $86.92$ & $89.76$ & $91.21$  \\ \hline 
& $\zeta_0$ &$73.44$ &$78.86$&$82.23$&$85.06$&$88.67$ \\ 
\rowcolor{Gray} \cellcolor{white}Euclidean & \cellcolor{white}  $\zeta_{0.5}$ &$73.27$ & $79.03$ & $82.65$ & $86.24$ & $88.98$ \\  
& $\zeta_{1}$ &$73.12$ & $78.92$ & $82.51$ & $86.01$ & $89.02$   \\
\hline 
\end{tabular}
\label{tab:odor_recon_acc}
\end{table}

In \Cref{fig:poincare_odor}, we show the embedded points in $\bI^2$ and $\R^2$ with different levels of allowable violated measurements. We can observe in \Cref{tab:odor_recon_acc} that hyperbolic space better represent the structure of olfactory data compared to Euclidean space of the same dimension. This is despite the fact that the number of measurements per variable is in favor of Euclidean embedding, and that the low rank approximation of hyperbolic Gramians is suboptimal. Moreover, if we remove a small number of outliers we can produce more accurate embeddings. These results corroborate the statistical analysis of Zhou \textit{et. al.} \cite{zhou2018hyperbolic} that aims to identify the geometry of the olfactory space. \footnote{Statistical analysis of Betti curve behavior of underlying clique topology \cite{giusti2015clique}.}  
\section{Conclusion}

We introduced hyperbolic distance matrices, an analogy to Euclidean distance matrices, to encode pairwise distances in the 'Loid model of hyperbolic geometry. Same as in the Euclidean case, although the definition of hyperbolic distance matrices is trivial, analyzing their properties gives rise to powerful algorithms based on semidefinite programming. We proposed a semidefinite relaxation which is essentially plug-and-play: it easily handles a variety of metric and non-metric constraints, outlier removal, and missing information and can serve as a template for different applications. Finally, we proposed a closed-form spectral factorization algorithm to estimate the point position from hyperbolic Gramians. Several important questions are still left open, most notably the role of the isometries in the 'Loid model and the related concepts such as Procrustes analysis.

\section*{Acknowledgement} 
\label{sec:ackowledgement}

We thank Lav Varshney for bringing our attention to hyperbolic geometry and for the numerous discussions about the manuscript.


\bibliographystyle{ACM-Reference-Format}
\bibliography{hdm}

\clearpage
\appendix
\section{Proof of Proposition \ref{prop:char_hyper_gram}}
\label{sec:proof_p1}

A hyperbolic Gramian can be written as $G = X^\T H X$ for a $X = [x_1, \ldots, x_N] \in \left( \bL^{d} \right)^N$. Let us rewrite it as
\begin{align*}
G &= \sum_{i=1}^{d} g_i g_i^\T  - g_0 g_0^\T \\
&= G^{+}-G^{-}
\end{align*}
where $g_i^\T$ is the $(i+1)$-th row of $X$, $G^{-} = g_0 g_0^\T$ and $G^{+} = \sum_{i=1}^{d} g_i g_i^\T$ are positive semidefinite matrices. We have $\rank{G^{-}} \leq 1$ and $\rank{G^{+}} \leq d$. On the other hand, we have
\begin{align*}
e_i^\T G e_j &\bydef [x_i, x_j] \\
&= -x_{0,i}x_{0,j} + \sum_{k=1}^{d} x_{k,i}x_{k,j} \\
&\stackrel{\text{(a)}}{=} -\sqrt{1+\norm{\bar{x}_i}^2} \sqrt{1+\norm{\bar{x}_j}^2} + \bar{x}_i^\T \bar{x}_j \\
&\stackrel{\text{(b)}}{\leq} -(1 +\bar{x}_i^\T \bar{x}_j ) + \bar{x}_i^\T \bar{x}_j  = -1.
\end{align*}
where $x_{k,i}$ is the $(k+1)$-th element of $x_i$, $\bar{x}_i = (x_{1,i}, \ldots, x_{d,i})^\T$, and $\text{(a)}$ is due to $\norm{x_i}_{H}^2 = \norm{x_j}_{H}^2 = -1$, and $\text{(b)}$ results from Cauchy-Shwartz inequality. The equality holds for $i = j$, which yields the $\diag{G} = -1$ condition.

Conversely, let $G = G^{+} - G^{-}$, where $G^{+}, G^{-} \succeq 0$, $\rank G^{-} \leq 1$, and $\rank G^{+} \leq d$. Let us write $G^{-} = g_0 g_0^\T$ and $G^{+}= \sum_{i=1}^{d} g_i g_i^\T$ for $g_0, \ldots, g_d \in \R^{N}$. Then, we define
\begin{equation*}
X \bydef  \left[ \begin{array}{c} g_0^\T \\ \vdots \\ g_d^\T \end{array}\right]  = [x_1, \cdots, x_N ] \in \R^{(d+1) \times N}.
\end{equation*}
where $x_n \in \R^{d+1}$ for all $n \in [N]$. By construction, we have $X^\T H X = G$, and
\begin{equation*}
\diag G = -1 \Rightarrow \norm{x_n}_{H}^{2} = -1, ~ \forall n \in [N].
\end{equation*}
Finally, $e_i^\T G e_j \leq -1$ guarantees that $x_n \in \bL^d$ for all $n \in [N]$. We prove the contrapositive statement. Let $x_i$ and $x_j$ belong to different the hyperbolic sheets, e.g. $x_i \in \bL^d, x_j \in -\bL^d$. Then,
\begin{align*}
e_i^\T G e_j &\bydef [x_i, x_j] \\
&= -x_{0,i}x_{0,j} + \sum_{k=1}^{d} x_{k,i}x_{k,j} \\
&\stackrel{\text{(a)}}{\geq} \sqrt{1+\norm{\bar{x}_i}^2} \sqrt{1+\norm{\bar{x}_j}^2} - \norm{\bar{x}_i} \norm{\bar{x}_j} \geq  0
\end{align*}
where $\text{(a)}$ is due to Cauchy-Shwartz inequality. This is in contradiction with $e_i^\T G e_j \leq -1$ condition. Therefore, $\set{x_n}$ belong to the same hyperbolic sheet, namely $\bL^d$.

\section{Derivations for Algorithm \ref{alg:RP}}

\begin{theorem}\label{prop:LGA}
Let $G \in \R^{N \times N}$ be a hyperbolic Gramian, with eigenvalue decomposition
\begin{equation}\label{eq:evd_G}
G = U^\T \Lambda U ,
\end{equation}
where $\Lambda = \diag{ ( \lambda_0, \cdots, \lambda_{N-1} )}$ such that 
\begin{itemize}
\item $\lambda_0 = \min_{i} \lambda_i$,
\item $\lambda_i$ is the $i$-th top element of $\set{\lambda_i}$ for $i \in \set{1, \cdots, d}$
\end{itemize}
The best rank-$(d+1)$ Lorentz Gramian approximation of $G$, in $\ell_2$ sense, is given by
\begin{equation*}
G_{d+1} = U_d^\T \Lambda_d U_d
\end{equation*}
where $\Lambda_d = \diag{[\lambda_0, u(\lambda_1), \cdots, u(\lambda_d)]}$, $u(x) = \max \set{x, 0}$, and $U_d \in \R^{(d+1) \times N}$ is the corresponding sliced eigenvalue matrix.
\end{theorem}
\begin{proof}
We begin by characterizing the eigenvalues of a Lorentz Gramian.
\begin{lemma} \label{lem:lorentz_gram_char}
Let $G \in \R^{N \times N}$ be a Lorentz Gramian of rank $d+1$ with eigenvalues $\psi_0 \leq \cdots \leq \psi_{d}$. Then, $\psi_0 < 0$, and $\psi_i > 0$, for $i \in \set{1, \cdots, d}$.
\end{lemma}
\begin{proof}
We write Lorentzian Gramian, $G = (g_{i,j})$, as $G = X^\T H X$ where
\begin{equation*}
X = [x_1, \cdots, x_N] \bydef  \left[ \begin{array}{c} g_0^\T \\ \vdots \\ g_d^\T \end{array}\right] \in  \R^{(d+1) \times N}.
\end{equation*}
Then, $G = G^{+} - G^{-}$ where $G^{+} \bydef \sum_{i=1}^{d} g_i g_i^\T$ is a positive semi-definite matrix of rank $d$ and with eigenvalues $0 < \gamma_1 \leq \cdots \leq \gamma_{d}$, and $-G^{-} \bydef -g_0 g_0^\T$ is a negative definite matrix of rank $1$, with eigenvalue $\mu \leq 0$. From Weyl's inequality \cite{horn2012matrix}, we have 
\begin{equation*}
\mu + \gamma_1 \leq \psi_0 \leq \mu + \gamma_{d}
\end{equation*}
where $\psi_0$ is the smallest eigenvalue of $G$. Therefore, $\psi_0$ can be non-positive (negative if $\mu + \gamma_{d} < 0$). For other eigenvalues of $G$, we have
\begin{equation*}
0+\gamma_1 \leq \psi_i \leq  \gamma_{d}, ~ \text{for} ~ 1 \leq i \leq d.
\end{equation*}
Hence, $\psi_i > 0$ for $i \in \set{1, \cdots, d}$. This is result is irrespective to the order of eigenvalues. 

Now, let us prove $\psi_0 < 0$.  Suppose $g_0 \in S = \mathrm{span} \set{g_i: i \in \set{1, \cdots, d}}$. Then, 
\begin{equation*}
\rank G = \rank \left[ \begin{array}{c} g_0^\T \\ \vdots \\ g_d^\T \end{array}\right] < d+1,
\end{equation*}
which is a contradiction. Therefore, we write $g_0 = \alpha t + \beta s$ where $s \in S$, $t \in S^{\perp}$ with $\norm{t} = 1$, $\alpha, \beta \in \R$ and $ \alpha \neq 0$. Then, we have
\begin{align*}
\psi_0 &\leq t^\T G t \\
&\stackrel{\text{(a)}}{=} - t^\T g_0 g_0^\T t \\
&= - \alpha^2 < 0
\end{align*}
where $\text{(a)}$ is due to $G = -g_0 g_0^\T + \sum_{i=1}^{d} g_i g_i^\T$ and $t \in S^{\perp}$. 
\end{proof}
Consider eigenvalue decomposition of $G$ in \cref{eq:evd_G}. Without loss of generality, we assume
\begin{itemize}
\item $\lambda_0 = \min_{i} \lambda_i < 0$,
\item $\lambda_i$ is the $i$-th top element of $\set{\lambda_i}$ for $i \in \set{1, \cdots, d}$. 
\end{itemize}
By construction $G = X^\T H X$ and from $\diag{G} = -1$ condition, we have
\begin{equation*}
\sum \lambda_i = -N.
\end{equation*}
Therefore, $\lambda_0 < 0$. From \Cref{lem:lorentz_gram_char}, one eigenvalue of a Lorentz Gramian is negative and the rest must be positive. Therefore, $\hat{G} = U_{d}^\T \Lambda_dU_{d}$ with eigenvalues  $\Lambda_d = \mathrm{diag} \set{\lambda_0, u(\lambda_1), \cdots, u(\lambda_d)}$ and eigenvectors $U_{d} = [u_{0}, \cdots, u_{d}]$, is the best rank-$(d+1)$ Lorentz Gramian approximation to $G$, i.e.
\begin{equation*}
\norm{\hat{G} - G}_{2}^2 = \inf_{H: \text{ Lorentz Gram. of rank } \leq d+1} \norm{H - G}_{2}^{2}.
\end{equation*}
\end{proof}
Finally, a rank-$(d+1)$ Lorentz Gramian with eigenvalue decomposition
\begin{equation*}
G_{d+1} = U_d \Lambda_d U_d^\T
\end{equation*}
can be decomposed as $X = R |\Lambda|^{1/2} U_d^\T \in \R^{(d+1)\times N}$ where $R$ is an arbitrary H-unitary matrix and $G_{d+1} = X^\T H X$.
\section{$\mathtt{Project}:\R^{d} \rightarrow \bL^d$} \label{sec:projection}
\begin{algorithm}[h]
\caption{Projection from $\R^{d+1}$ to $\bL^d$}\label{alg:projection}
\begin{algorithmic}[1]
\Procedure{$\mathtt{Project}$}{$x$}
\State For $x \in \R^{d+1}$, let
\begin{equation*}
\hat{x} = \begin{cases}
(1, 0^\T)^\T & x \in \set{(x_0, 0^\T)^\T: x_0 \leq 2}, \\
(\frac{1}{2}x_0, \hat{x}_1, \cdots, \hat{x}_d)^\T &x \in \set{(x_0, 0^\T)^\T: x_0 > 2} \\
 &\text{and for a }(\hat{x}_1, \cdots, \hat{x}_d) \in S, \\
x(\lambda) &\text{otherwise and for }\lambda: \\
 & \norm{x(\lambda)}_{H}^{2} = -1.
\end{cases}
\end{equation*}
where  $x(\lambda) = (I+\lambda H)^{-1}x$ and
\begin{equation*}
S = \set{(x_1, \cdots, x_d): x_1^2 + \cdots + x_d^2  = -1 + \frac{1}{4}x_0^2}.
\end{equation*}
\State \textbf{return} $\hat{x}$.
\EndProcedure
\end{algorithmic}
\end{algorithm}
\begin{proof}
Let us reformulate the following projection problem
\begin{equation} \label{eq:projection}
\hat{x} \in \argmin_{y \in \bL^d} \norm{y - x}^2
\end{equation}
as unconstrained augmented Lagrangian minimization, i.e.
\begin{equation*}
L(y, \lambda) = \norm{y - x}^2 + \lambda(y^\T Hy + 1).
\end{equation*}
The first order necessary condition for $\hat{x}$ to be a (local) minimum of \cref{eq:projection} is
\begin{equation}\label{eq:fo_optimality}
(I + \lambda^* H)\hat{x} =  x
\end{equation}
for a $\lambda^* \in \R$ such that $\hat{x} \in \bL^d$.  

$\lambda^* = -1$: This happens when $x = (x_0, 0^\T)^\T$ and $x_0 \geq 2$. Following from optimality condition of \cref{eq:fo_optimality} and $\norm{\hat{x}}_{H}^2 = -1$, we have $\hat{x} = (\frac{1}{2}x_0, \hat{x}_1, \cdots, \hat{x}_d)^\T$, where 
\begin{equation*}
\hat{x}_1^2 + \cdots + \hat{x}_d^2  = -1 + \frac{1}{4}x_0^2.
\end{equation*}
Therefore, $\hat{x}$ could be any point on a $(d-1)$-dimensional sphere on $\bL^d$. For $x = (x_0, 0^\T)^\T$ and $x_0 \leq 2$, we have $\hat{x} = (1, 0^\T)^\T$.

$\lambda^* = 1$: This happens for $x = (0, x_1, \cdots, x_d)^\T$. From optimality condition of \cref{eq:fo_optimality}, we have $\hat{x} = (\hat{x}_0, \frac{1}{2}x_1, \cdots, \frac{1}{2}x_d)$, where $\hat{x}_0 = \frac{1}{2}\sqrt{ x_1^2 + \cdots + x_d^2 +4}$. 

For non-degenerate cases of $\lambda^{*} \neq \pm 1$, we have 
\begin{equation}\label{eq:x_hat}
\hat{x} =  (I + \lambda^* H)^{-1}x,
\end{equation} 
where $\lambda^* \in \set{\lambda : \norm{(I + \lambda H)^{-1} x}_{H}^2 = -1, \hat{x}_0 \geq 0}$. 

$(1)$ $\lambda^* \in (-1 , 1)$: First, we define 
\begin{equation*}
f(\lambda) = \norm{(I + \lambda H)^{-1} x}_{H}^2.
\end{equation*}
This is a monotonous function on $(-1, 1)$, with $\lim_{\lambda \rightarrow 1^{-}} f(\lambda) = -\infty$, and  $\lim_{\lambda \rightarrow -1^{+}} f(\lambda) = +\infty$. Hence, $f(\lambda) = -1$ has a unique solution $\lambda^* \in (-1, 1)$. Finally, $\hat{x}$ is a local minima since the second order sufficient condition
\begin{equation*}
I + \lambda^* H \succ 0
\end{equation*}
is satisfied for $\lambda^* \in (-1, 1)$. Lastly, from \cref{eq:x_hat}, we have $\hat{x}_0 x_0 \geq 0$. In other words, $\lambda^* \in [-1 , 1]$ if and only if $x$ is in the same half-space as $\bL^d$, i.e. $x_0 \geq 0$.

$(2)$ $\lambda^* \in (-\infty, -1)$: Similarly, $f(\lambda)$ is a continuous function in this interval with
$\lim_{\lambda \rightarrow -1^{-}} f(\lambda) = +\infty$, $\lim_{\lambda \rightarrow -\infty} f(\lambda) = 0$, and its first order derivative 
\begin{equation*}
\frac{d}{d \lambda} f(\lambda) =  - \frac{2}{(1-\lambda)^3}x_0^2- \frac{2}{(1+\lambda)^3}\sum_{i=1}^{d} x_i^2 
\end{equation*}
has at most one zero. Therefore, $f(\lambda) = -1$ has at most two solutions. The second order necessary condition for local minima is $v^\T ( I + \lambda^* H ) v \geq 0$ for all $v \in T_{\hat{x}} \bL^d$, where
\begin{equation*}
T_{\hat{x}} \bL^d = \set{v \in \R^{d+1} : x^\T (I + \lambda^* H)^{-1} H v = 0}.
\end{equation*}
However, there exists a $v \in T_{\hat{x}} \bL^d$ where $v = (0,\bar{v}^\T)^\T$ which violates the second order necessary condition, $v^\T ( I + \lambda^* H ) v < 0$. Therefore, $\hat{x}$ -- even if it exists -- is not a local minima.

$(3)$ $\lambda^* \in (1, \infty)$: We can easily see that $\lim_{\lambda \rightarrow 1^{+}} f(\lambda) = -\infty$, $\lim_{\lambda \rightarrow +\infty} f(\lambda) = 0$, and $\frac{d}{d \lambda} f(\lambda) = 0$ has at most one solution in this interval. Therefore, $f(\lambda) = -1$ has exactly one solution. However, we have $\hat{x}_0 x_0 \leq 0$ from \cref{eq:x_hat}. In other words, $\lambda^* \in (1, \infty)$ if and only if $x$ is in the opposite half-space of $\bL^d$, i.e. $x_0 \leq 0$. Finally, $\hat{x}$ is the unique minima, since the projection of $x \notin S$ to the closed and convex set of
\begin{equation*}
S = \set{x: x_0 \geq 0, \norm{x}^{2}_{H} \leq -1}
\end{equation*}
always exits and is unique.
\end{proof}

\section{Proof Outline of Proposition \ref{prop:spectral_factorization}}
\label{sec:proof_p2}

Let $X = R |\Lambda|^{1/2} U^\T$. Then, 
\begin{align*}
X^\T H X &= U |\Lambda|^{1/2} R^\T H R |\Lambda|^{1/2} U^\T \\
&\stackrel{\text{(a)}}{=} U |\Lambda|^{1/2} H |\Lambda|^{1/2} U^\T \\
&\stackrel{\text{(b)}}{=}  G
\end{align*}
where $\text{(a)}$ is due to properties of H-unitary matrices, $\text{(b)}$ from $|\lambda_0|^{1/2}(-1)|\lambda_0|^{1/2} =\lambda_0$ for $\lambda_0 \leq 0$. Therefore $X= R |\Lambda|^{1/2} U^\T$ is a hyperbolic spectral factor of $G$. Finally, the uniqueness of these factors is due to fact that $H$-unitary operators fully characterize isometries in the 'Loid model \cite{cannon1997hyperbolic}.
\end{document}